\documentclass[]{easychair}

\usepackage[utf8]{inputenc}
\usepackage[T1]{fontenc}
\usepackage{booktabs}
\usepackage{url}
\usepackage{xspace}
\usepackage{microtype}
\usepackage{cleveref}
\usepackage{enumitem}
\usepackage{comment}

\usepackage{todonotes}
\usepackage{amsthm}
\usepackage{amssymb}
\newtheorem{theorem}{Theorem}
\newtheorem{corollary}[theorem]{Corollary}

\newtheorem{lemma}[theorem]{Lemma}
\newtheorem{proposition}[theorem]{Proposition}

\theoremstyle{definition}
\newtheorem{definition}[theorem]{Definition}
\newtheorem{example}[theorem]{Example}

\usepackage{macros}

\title{Fuzzy Datalog$^\exists$ over Arbitrary t-Norms}

\author{
  Matthias Lanzinger\inst{3,1}
  \and
  Stefano Sferrazza\inst{1,3}
  \and
  Przemysław A. Wałęga\inst{1}
  \and
  Georg Gottlob\inst{2,1}
 }

\institute{
  University of Oxford
  \and
  University of Calabria
  \and
  TU Wien
}

\authorrunning{Lanzinger, Sferrazza, Wałęga \& Gottlob}

\titlerunning{Fuzzy Datalog$^\exists$ over Arbitrary t-Norms}

\begin{document}
  
\maketitle             

\begin{abstract}
One of the main challenges in the area of Neuro-Symbolic AI is to perform logical reasoning in the presence of 
both neural and symbolic data. This requires combining heterogeneous data sources such as knowledge graphs, neural model predictions, structured databases, crowd-sourced data, and many more. To allow for such reasoning, we generalise the standard rule-based language Datalog with existential rules (commonly referred to as tuple-generating dependencies) to the fuzzy setting, 
by allowing for arbitrary $t$-norms in the place of classical conjunctions in rule bodies.
The resulting formalism allows us to perform reasoning about data associated with degrees of uncertainty while preserving computational complexity results and the applicability of reasoning techniques established for the standard Datalog setting. In particular, we provide fuzzy extensions of Datalog chases
which produce fuzzy universal models and we exploit them to show that in important fragments of the language, reasoning has the same complexity as in the classical setting.
\end{abstract}

\section{Introduction}\label{intro}

We currently see a rapid growth of Artificial Intelligence and its usage in large-scale applications, such as image and speech recognition, knowledge graphs completion, or recommendation generation.
Such systems produce  huge amounts of data, whose facts are associated with
degrees of certainty---expressing the level of confidence in the truth of
the datum. 
Reasoning about such data gives rise to new challenges for data management.
Specifically, there is a growing demand for
logical reasoning methods capable of integrating precise and uncertain data gathered from heterogeneous sources.
Developing efficient approaches for this task would allow us to make a significant step towards a tight integration of symbolic and sub-symbolic AI.

This research direction is currently intensively studied within the areas of Neural-Symbolic AI
\cite{garcez2019neural,garcez2023neurosymbolic} and 
Statistical-Relational AI
\cite{de2020statistical} which, in the last years, gave rise to numerous formalisms aiming to integrate various aspects of logical reasoning with neural models.
A~number of approaches are based on
combining logic programming languages with probabilistic and neural predicates; representative examples in this class are
DeepProbLog~\cite{DBLP:journals/ai/ManhaeveDKDR21},
SLASH~\cite{skryagin2022neural}, 
NeurASP~\cite{yang2023neurasp}, and
Generative Datalog~\cite{alviano2023generative}.
There are also approaches, like Logic Tensor Networks (LTN) \cite{luciano2022logic,serafini2016learning} or its extension LYRICS \cite{marra2019lyrics}, which propose to adapt logical semantics to the neural setting by interpreting terms with tensors and connectives
with $t$-norms.
On the other hand, there is  a long-standing research on fuzzy logics \cite{DBLP:books/kl/Hajek98}
and their  applications to
logic programming~\cite{achs1995fuzzy,DBLP:journals/fss/Ebrahim01,DBLP:journals/tnn/EklundK92,DBLP:journals/tfs/IranzoS18,lanzinger2022mvdatalog,DBLP:conf/epia/MedinaOV01} and description logics~\cite{lukasiewicz2008managing,straccia2001reasoning,bobillo2008fuzzydl}, among others.
More recently, the fuzzy setting is studied for complex multi-adjoint~\cite{DBLP:journals/fss/CornejoLM18,DBLP:conf/epia/MedinaOV01} and Prolog-derived semantics based on fuzzy similarity of constants and fuzzy unification procedures~\cite{DBLP:journals/tfs/IranzoS18}.

Although the recent progress on logical formalisms for reasoning about neural and uncertain data is undeniable, current approaches still do now allow to fully address the grand challenge of integrating neural data with logical reasoning
\cite{bardin2023machine,amel2019shallow}.
Indeed, current methods
are often of high computational complexity (reasoning is  often undecidable and sometimes complexity  is not even analysed), require completely new, often exotic, reasoning procedures tailored to the  introduced formalisms, or
impose significant restrictions on the allowed forms of uncertainty as well as on their interaction within logical reasoning.

We address these difficulties by introducing an extension \tdat of the standard rule-based language Datalog (with existential rules)
to the setting where data can be associated with degrees of certainty and rules are equipped with a wide range of connectives (interpreted by arbitrary $t$-norms) operating on  
these degrees.
To illustrate the reasoning capabilities of \tdat consider the example from \Cref{fig:ex}, where the task is to 
determine a common hypernym of objects presented in images $\mathit{img1}$ and $\mathit{img2}$.
To this end, we apply the CNN image classifier EfficientNet~\cite{DBLP:journals/ai/ManhaeveDKDR21}, which provides us with predicted labels for the subject of the image and truth degrees of these predictions.
This allows us to produce fuzzy facts with  \emph{neural predicates};
the highest certainty degrees $0.800$ and $0.900$ are associated with
$\mathit{NeuralLabel(img1,tiger\_shark)}$ and
$\mathit{NeuralLabel(img2,tench)}$, respectively.
We also use a lexical database WordNet\footnote{\url{http://wordnetweb.princeton.edu/}} \cite{DBLP:journals/cacm/Miller95} which contains, among many others, precise facts about hypernyms;
for example we obtain facts         $\mathit{Hypernym(tiger\_shark, fish)}$ and         $\mathit{Hypernym(tench, fish)}$.
To perform reasoning based on a combination of neural data  from EfficientNet and precise facts from WordNet we use a \tdat program consisting of the following rules, where conjunctions in rule bodies are replaced with operators corresponding to $t$-norms:
\begin{align}
    \mathit{NeuralLabel}(x,y) & \to \mathit{Class}(x,y), \tag{$r_1$}\label{ex1}
    \\ 
    \mathit{Class}(x,y) \conjL \mathit{Hypernym}(y,z)
    & \rightarrow \mathit{Class}(x,z),  \tag{$r_2$}\label{ex2}
    \\ 
    \mathit{Class}(x,z) \conj_{\product} \mathit{Class}(y,z)
    & \rightarrow \mathit{CommonClass}(x,y,z). \tag{$r_3$}\label{ex3}
\end{align}
Rule \eqref{ex1} introduces a binary predicate $\mathit{Class}$ which holds for image labels predicted by EfficientNet via $\mathit{NeuralLabel}$,
Rule \eqref{ex2} exploits knowledge about $\mathit{Hypernym}$ from WordNet to assign image objects to classes, and
Rule \eqref{ex3} derives common classes for pairs of images.
Note that \tdat allows for using various $t$-norms in different rules; in particular, we use Łukasiewicz $t$-norm $\conjL$ and the product $t$-norm $\conj_{\product}$, which operate on certainty degrees as follows: $a \conjL b = \max\{0, a+b-1\}$ and $a \conj_{\product} b = a \cdot b$.
Thus, Rule \eqref{ex2} allows us to derive $\mathit{Class(img1,fish)}$ with certainty $0.800 \conjL 1 = 0.800$ and  $\mathit{Class(img2,fish)}$ with certainty $0.900 \conjL 1 = 0.900$ (facts about $\mathit{Hypernym}$ are precise and so, they have a truth degree 1).
Then, Rule \eqref{ex3} derives $\mathit{CommonClass}(img1,img2,fish)$ with certainty $0.800 \conj_{\product} 0.900 = 0.720$. 
Note that we can derive also $\mathit{CommonClass(img1,img2,tiger\_shark)}$ with a low certainty $0.016$, since EfficientNet 
misclassified  image $img2$ as a $\mathit{tiger\_shark}$ with certainty $0.020$.

\pgfmathsetmacro{\rectw}{1.9}
\pgfmathsetmacro{\recth}{8}
\pgfmathsetmacro{\gap}{0.8}

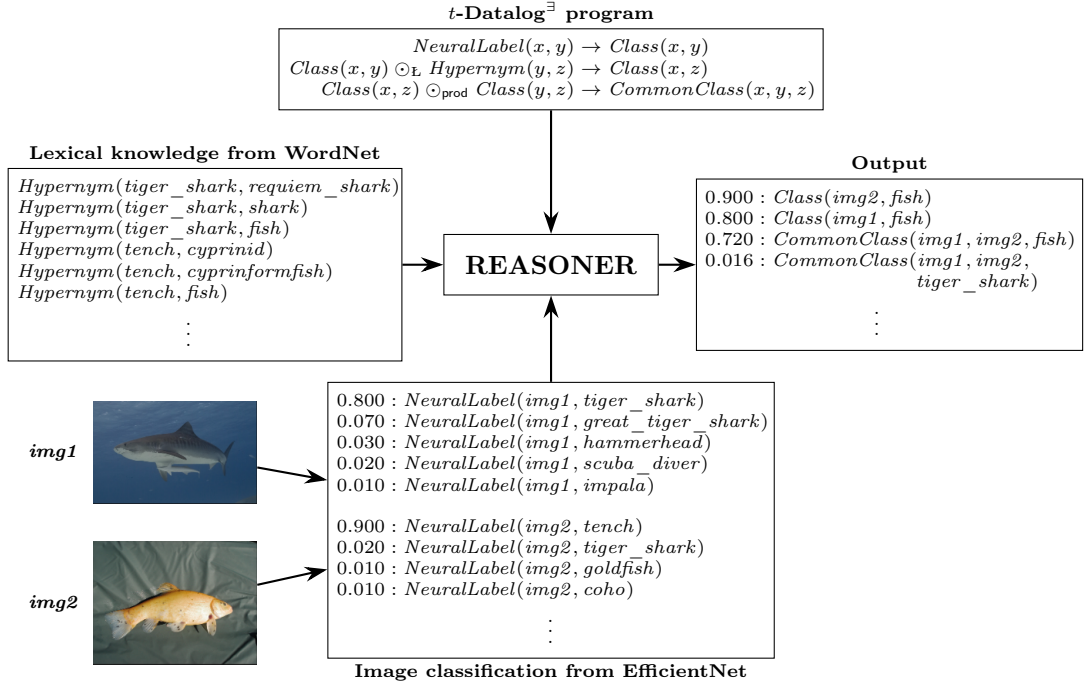
\begin{figure*}[ht!]
\begin{tikzpicture}%

\setlength{\jot}{0pt} 
\usetikzlibrary{arrows.meta,arrows}
\tikzstyle{every node}=[font=\scriptsize]
\tikzset{>=latex}

\node[inner sep=0pt] (shark) at (-5,-2.5)
    {\includegraphics[width=.15\textwidth]{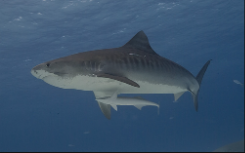}};

\node[inner sep=0pt] (tench) at (-5,-4.5)
    {\includegraphics[width=.15\textwidth]{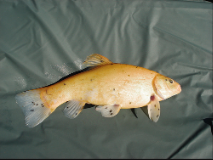}};

    \node[rectangle, draw, text width=7cm] at (0,2.6) (program) {
        \begin{minipage}{\textwidth}
            $\begin{aligned}
    \mathit{NeuralLabel}(x,y) & \to \mathit{Class}(x,y)
    \\
    \mathit{Class}(x,y) \conjL \mathit{Hypernym}(y,z)
    & \rightarrow \mathit{Class}(x,z)
    \\
    \mathit{Class}(x,z) \conj_{\product} \mathit{Class}(y,z)
    & \rightarrow \mathit{CommonClass}(x,y,z)
            \end{aligned}$
        \end{minipage}
    };
  
\node[rectangle, draw, text width=5.7cm] at (0,-3.4) (data1) {
        \begin{minipage}{\textwidth}
            $\begin{aligned}
        &0.800: \mathit{NeuralLabel(img1,tiger\_shark)}\\
        &0.070: \mathit{NeuralLabel(img1,great\_tiger\_shark)}\\
        &0.030: \mathit{NeuralLabel(img1,hammerhead)}\\
        &0.020: \mathit{NeuralLabel(img1,scuba\_diver)}\\ 
        &0.010: \mathit{NeuralLabel(img1,impala)}\\[2.5ex]      
        &0.900:  \mathit{NeuralLabel(img2,tench)}\\
        &0.020: \mathit{NeuralLabel(img2,tiger\_shark)}\\
        &0.010:  \mathit{NeuralLabel(img2,goldfish)}\\
        &0.010: \mathit{NeuralLabel(img2,coho)}\\ 
        & \qquad \qquad \qquad \qquad \qquad   \vdots
            \end{aligned}$
        \end{minipage}
    };

\node[rectangle, draw, text width=5cm] at (-4.6,0) (data2) {
        \begin{minipage}{\textwidth}
            $\begin{aligned}
        &\mathit{Hypernym(tiger\_shark, requiem\_shark)}\\
        &\mathit{Hypernym(tiger\_shark, shark)}\\
        &\mathit{Hypernym(tiger\_shark, fish)}\\        
        &\mathit{Hypernym(tench, cyprinid)}\\      
        &\mathit{Hypernym(tench, cyprinformfish)}\\     
        &\mathit{Hypernym(tench, fish)}\\             
        & \qquad \qquad \qquad \qquad   \vdots
            \end{aligned}$
        \end{minipage}
    };

\node[rectangle, draw, inner sep=8pt] at (0,0) (system) {\normalsize \textbf{REASONER}};

\node[rectangle, draw, text width=4.9cm] at (4.5,0) (out) {
        \begin{minipage}{\textwidth}
            $\begin{aligned}
        &0.900: \mathit{Class(img2,fish)}\\
        &0.800:  \mathit{Class(img1,fish)}\\
        &0.720:  \mathit{CommonClass(img1,img2,fish)}\\
         &0.016:  \mathit{CommonClass(img1,img2},\\
         & \qquad \qquad \qquad \qquad \qquad \mathit{tiger\_shark)}\\
        & \qquad \qquad \qquad  \qquad \vdots
            \end{aligned}$
        \end{minipage}
    };

\node[below=-0.05cm of data1] {\textbf{Image classification from EfficientNet}};
\node[above=-0.05cm of program] {\textbf{\tdat  program}};
\node[above=-0.05cm of out] {\textbf{Output}};
\node[above=-0.05cm of data2] {\textbf{Lexical knowledge from WordNet}};
\node[left=0.05cm of shark] {\textbf{\textit{img1}}};
\node[left=0.05cm of tench] {\textbf{\textit{img2}}};

\draw[-{Stealth[length=3mm, width=2mm]},  thick] (shark) -- (data1); 
\draw[-{Stealth[length=3mm, width=2mm]},  thick] (tench) -- (data1); 
\draw[-{Stealth[length=3mm, width=2mm]},  thick] (program) -- (system); 
\draw[-{Stealth[length=3mm, width=2mm]},  thick] (data1) -- (system); 
\draw[-{Stealth[length=3mm, width=2mm]},  thick] (data2) -- (system); 
\draw[-{Stealth[length=3mm, width=2mm]},  thick] (system) -- (out); 

\end{tikzpicture}
\caption{An application of \tdat to derive common classes of objects in input images $\mathit{img1}$ and $\mathit{img2}$;
reasoning is performed based on image classifications from EfficientNet (with uncertainty degrees) and lexical knowledge from WordNet (precise information).}\label{fig:ex}
\end{figure*} 
Our formalism \tdat is, therefore, a natural extension of \datex (Datalog extended with existential rules \cite{baget2011rules}, also known as tuple-generating dependencies \cite{beeri1981implication} and studied under the name of \datalogpm \cite{cali2009general,DBLP:conf/kr/GottlobLP14}) to the fuzzy setting, where conjunctions are replaced with $t$-norms.
It is worth emphasising that the class of \tdat programs is very broad; we allow for existential quantification in rule heads (which we did not use in the exemplary program to simplify presentation), any arity predicates, recursion, and arbitrary $t$-norms (in contrast to many other approaches, e.g., our previous research tailored specifically to the Łukasiewicz $t$-norm~\cite{lanzinger2022mvdatalog}).
In particular, no restriction on the choice of $t$-norms allows us to model a  range of interactions between degrees of certainty. 
This flexibility
is also important from the practical perspective, as the choice of $t$-norms has a significant impact on the performance of a system~\cite{farahbod2012comparison}.

The main advantage, 
distinguishing \tdat  from the related formalisms, 
is  that \tdat does not only allow us to perform complex logical reasoning about uncertain data gathered from heterogeneous data sources, but il also allows us to apply 
well-studied \datex reasoning mechanisms. 
In particular, 
we show in the paper how to adapt the (semi-oblivious and restricted) chase procedures developed for \datex and we prove that complexity of reasoning in various fragments of \tdat is the same as in the corresponding fragments of \datex.
Hence, we obtain a proper extension of \datex which allows us to perform complex reasoning about degrees of uncertainty with no negative impact on the computational complexity and with the possibility to use standard chase procedures.

The main contributions of this paper are as follows\footnote{Preliminary ideas for \tda (i.e., without existential quantification) were previously presented as an extended abstract~\cite{lanzinger2022new}.} :

\begin{itemize}[noitemsep]
    \item We introduce   \tdat (\Cref{sec:datex}) as a fuzzy extension of  \datex allowing for arbitrary $t$-norms instead of standard Boolean conjunction in rule bodies. Both syntax and semantics are defined by natural extensions of \datex to the fuzzy setting.
    As a result, we lay foundations for fuzzy extensions of the Datalog$^\pm$ family of ontology languages (obtained by imposing various restrictions on \datex) suitable for neuro-symbolic applications.

    \item We propose (\Cref{sec:chasetech}) a fuzzy version of the chase and we show  (\Cref{sec:universal}) that, similarly as in \datex, application of a finite fuzzy chase
    results in a fuzzy universal model, which can be used to decide entailment. 
        We observe, however, that reasoning with fuzzy chases introduces new challenges, which disallows us to directly translate results on termination and complexity from \datex.
    \item 
We introduce (\Cref{{sec:termination}}) a new type of \emph{truth-greedy} fuzzy chases for \tdat.
We show their relation to standard chases for \datex, which allows us to exploit termination and complexity results for \datex.
In particular, 
we show \ptime-completeness for entailment in  \tda and in weakly acyclic \tdat, matching the complexity of Datalog and weakly acyclic \datex.

    \item 
    We introduce (\Cref{sec:negation}) an extension of  \tdat with  fuzzy negations and any  other unary operators (e.g., threshold operators) interpreted as functions  $[0,1] \longrightarrow [0,1]$.  
    We show that adding such operators does not increase complexity of reasoning if the input program is stratifiable (and does not use existential quantification), namely it remains \ptime-complete in data complexity.
\end{itemize}

\section{\datex over $t$-norms}
\label{sec:datex}
In this section, we introduce syntax and semantics of  \tdat, as well as present the basic reasoning problem and notation which we will use in the paper.

\paragraph{Signature and domain.} 
We  fix  a signature $\sigma$ (i.e. a set of predicate symbols, each of a fixed, but arbitrary arity) and countable sets  $\dom$ and \nulls{} of \emph{object elements} and \emph{nulls}, resepectively, with  $\dom \cap \nulls = \emptyset$;
we let $\domn = \dom \cup \nulls$.
We will use  $\gatom$ and $\gatomn$ for the sets of all ground atoms with constants in  $\dom$ and in $\domn$, respectively.

\paragraph{Fuzzy dataset.} 
A \emph{fuzzy dataset}  is a
partial function
$\D: \gatom \longrightarrow (0,1]$
assigning  real numbers from the interval $(0,1]$---treated as \emph{truth degrees}---to a \emph{finite} number of
 atoms in $\gatom$. 
Note that a fuzzy dataset does not assign truth degrees to atoms with nulls and that
it never assigns 0 to any atom,
which is in line with the standard definition of a dataset listing facts which need to hold true (but not mentioning which facts need to hold false).

\paragraph{$t$-norms.} 
A \emph{$t$-norm} is any commutative, monotone, and associative function of the form $\conj \colon [0,1] \times [0,1] \longrightarrow [0,1]$, 
with 1 being the identity element~\cite{DBLP:books/kl/Hajek98}. 
It follows from the definition that if $a$ and $b$ are Boolean (i.e., $0$ or $1$), then
$a \conj b$ coincides with the value of the standard conjunction
$a \land b$, for any  $t$-norm $\conj$.
Thus, $t$-norms provide a  generalisation of the standard conjunction to the fuzzy setting. 
Commonly used $t$-norms in fuzzy logics include the \emph{minimum  $t$-norm} (also known as G\"odel $t$-norm) $a \conj_{\min} b = \min \{a,b\} $,  \emph{\luka $t$-norm} $a \conjL b = \max\{0, a+b-1\}$, and the \emph{real product $t$-norm} $a \conj_{\product} b = a \cdot b$. 
However, $t$-norms can  be significantly more complex, for example, all functions of the form $f_p(a,b) = (a^p+b^p-1)^{\frac{1}{p}}$, for $p<0$, are $t$-norms (part of the Schweizer-Sklar family of $t$-norms~\cite{DBLP:journals/chinaf/ZhangHX06}). 
The following known observation will be particularly important for our results.
\begin{proposition}
\label{minislargest}
For every $t$-norm \conj{} and all $a,b\in [0,1]$ it holds that  $a \conj b \leq a \conj_{\min} b$.
\end{proposition}
\noindent For an  overview of $t$-norms and their properties, refer  to the work of Klement et al.~\cite{tnormsbook}.

\paragraph{\tdat programs.} 
A \tdat \emph{program} $\prog$ is a finite set of \emph{rules} $r$, of the form
\begin{align*}
  R_1(\mathbf{x_1}) \conj_r \cdots \conj_r R_\ell(\mathbf{x_\ell}) & \rightarrow \exists \mathbf{z}\, S(\mathbf{x_h}, \mathbf{z}) ,
\end{align*}
where $\conj_r$ is any $t$-norm\footnote{We slightly abuse notation by using the same symbol, $\conj_r$, for a connective in a rule and 
for a $t$-norm which interprets this connective.} (each rule can use a different $t$-norm),
$S, R_1, \dots, R_\ell$ are  predicate symbols in the signature and  $\mathbf{x_1}, \dots, \mathbf{x_\ell}, \mathbf{x_h}, \mathbf{z}$ are (possibly empty)  sequences of variables of arities matching the predicate symbols. 
The left-hand side of the implication $\rightarrow$ in a rule $r$ is the \emph{body}, $\body(r)$,  and the right-hand side is its \emph{head}, $\head(r)$. 
Variables in $\mathbf{x_1}, \dots, \mathbf{x_\ell}$ are called \emph{body variables},
those in $\mathbf{x_h}$  \emph{frontier variables},
$\fr(r)$, 
and those in $\mathbf{z}$ are  \emph{existential variables}.
We assume that in each rule
the set of body variables contains each frontier variable, but does not contain any existential variable.
A program with no existential variables is  called  a \tda program. 
We use  $\vars(\phi)$ to refer to the set of all variables in an expression $\phi$ (e.g., in a rule or in a rule body).
A grounding of a rule $r$
is any function $\rho: \vars(\body(r)) \longrightarrow \domn$ 
assigning domain constants and nulls to  body variables.

\paragraph{Semantics.}
A \emph{fuzzy interpretation} is a function $\inter \colon \gatomn \longrightarrow [0,1]$, assigning truth degrees to 
ground atoms (also to atoms with nulls).
The function is extended to complex expressions $\phi$ and $\phi'$ inductively as follows:
  \begin{align*}
  \inter(\phi \conj \phi') & = \conj(\inter(\phi), \inter(\phi')) ,
  \\
  \inter(\phi \rightarrow \phi')& = \min \left\{1, 1 - \inter(\phi) + \inter(\phi') \right\} ,
  \\
   \inter \big(\exists \mathbf{z} \; S ( \mathbf{a}, \mathbf{z}) \big) & = \sup \left\{ \inter \big( S(\mathbf{a}, \mathbf{b}) \big) \mid \mathbf{b} \in \domn^{|\mathbf{z}|} \right\} ,
  \end{align*}
where $\conj$ is any $t$-norm and $\mathbf{a}$ is a sequence of constants in $\domn$.
Note that we use the same  
Łukasiewicz semantics for implication $\to$
in all rules.

A rule $r$ is \emph{$K$-satisfied} by a fuzzy interpretation $\inter$,
for a rational number $K \in [0,1]$,
if
$\inter(\rho(r)) \geq K$
for every grounding $\rho$ of $r$. 
Thus,  $\phi \to \phi'$ is $K$-satisfied if
$ \inter(\rho(\phi')) - \inter(\rho(\phi)) \geq K -1$.
In particular, 
$1$-satisfiability  requires that 
$ \inter(\rho(\phi')) - \inter(\rho(\phi)) \geq 0$,
that is, the truth degree of the head is at least as large as the truth degree of the body.
On the other hand, 
$\phi \to \phi'$ is trivially $0$-satisfied by any $\inter$, since 
$ \inter(\rho(\phi')) - \inter(\rho(\phi))   \geq - 1$ is always true.
A fuzzy interpretation  $\inter$ is a \emph{$K$-fuzzy model} of a program $\prog$ if all rules in $\prog$ are $K$-satisfied by $\inter$.
For a fuzzy dataset $\D$, we let $\inter_\D$ be its  
minimal fuzzy interpretation defined such that $\inter_\D(A)=\D(A)$ if $A$ is in the domain of $\D$  and $\inter_\D(A)=0$ otherwise.
Then, $\inter$ is
a \emph{fuzzy
model}  of $\D$
if $\inter(A) \geq  \I_\D(A)$ for every ground atom $A$.

\paragraph{Reasoning and complexity.}

The basic reasoning problem we consider 
is $K$-entailment, for any $K \in [0,1]$, which is to check whether in every $K$-fuzzy model of a  program $\prog$ and a fuzzy dataset $\D$,
the truth degree of a goal ground atom $G$ is not smaller than a target value $c$.
Hence, the problem is defined as follows:

\begin{problem}{$K$-Entailment}
Input: & A \tdat program $\prog$, a fuzzy dataset $\D$, a ground atom $G \in \gatom$, and $c \in[0,1]$.\\
Output: & `Yes' if and only if $\inter(G) \geq c$, for all $K$-fuzzy models $\inter$ of $\prog$ and $\D$.
\end{problem}
\noindent 
If the output is `yes', we  say that 
$G$ is $(c,K)$-\emph{entailed} by $\prog$ and $\D$, written as $\PD \ent{c}{K} G$.

Clearly, checking $K$-entailment requires applying $t$-norms and, although
computing most of the standard $t$-norms is easy (e.g., all the $t$-norms mentioned in this paper),
one can introduce computationally-demanding $t$-norms. 
For this reason,
we will abstract away from the complexity of computing $t$-norms.
In particular, 
when studying  computational complexity of $K$-\textsc{Entailment},
we will treat
the time required to compute $t$-norms 
and the memory used to store truth degrees as constant.
Moreover, we will  focus on the \emph{data complexity} of $K$-\textsc{Entailment}---which is  with respect to the size of $\D$ only---and is particularly important for data-intensive applications, like the one from \Cref{fig:ex}.
Furthermore, for the sake of simplicity, we only explicitly analyse atomic entailment in this paper. Recall that entailment of a conjunctive query $\phi$ can be expressed by adding a rule $\phi \rightarrow \mathit{Goal}$ and checking  entailment of $\mathit{Goal}$.
Therefore, complexity results in our paper transfer also to conjunctive query entailment.

We can  observe that $K$-\textsc{Entailment} extends the standard notion of entailment in \datex. 
In particular,  
if $c=K=1$ and the only truth degree assigned by a fuzzy dataset is 1, then $K$-\textsc{Entailment} coincides with the standard entailment in \datex, as shown below.

\begin{theorem}
\label{thm:all1}
Let $\prog$ be a \tdat program, $\D$ a fuzzy dataset, and  $G \in \gatom$.
If the only truth degree assigned by $\D$ is 1 (i.e., $\D(A) = 1$ whenever  $\D(A)$ is defined), then the following are equivalent:
\begin{enumerate}
\item $\PD \ent{1}{1} G$,

\item 
$G$ is entailed by the \datex counterparts
$\prog'$ and $\D'$ of $\prog$ and $\D$ (namely 
$\prog'$ is obtained by replacing $t$-norms in $\prog$  with conjunctions and $\D' = \{A \mid \D(A) = 1 \}$).  
\end{enumerate}
\end{theorem}
\begin{proof}
Assume that $\PD \ent{1}{1} G$ and let  $S \subseteq \gatomn$ be a \datex model of $\prog'$ and $\D'$.
We will show that $G \in S$.
We define a fuzzy interpretation $\inter$ 
such that for  any $A \in \gatomn$ we have  $\inter(A)=1$   if $A \in S$, and $\inter(A)=0$ if $A \notin S$.
Hence, $\inter$ is a fuzzy model of $\D$.
We claim that $\inter$ is also a 1-fuzzy model of $\prog$.
For this, it suffices to show that each rule of $\prog'$ is satisfied in $S$ if and only if its counterpart in $\prog$ is 1-satisfied in $\inter$.
This  holds since
the semantics of any $t$-norm coincides with the semantics of a conjunction and the semantics of $\to$ in \tdat coincides with its semantics in \datex,
whenever a fuzzy interpretation assigns only Boolean truth degrees to all atoms.
As $\PD \ent{1}{1} G$ and $\inter$ is a 1-fuzzy model of $\prog$ and $\D$, we obtain that $\inter(G)=1$, and so, $G \in S$, as required.

For the opposite direction assume that 
$\prog'$ and $\D'$ entail $G$, and let
$\inter$ be a 1-fuzzy model of $\prog$ and $\D$.
We construct a `crisp' version $\inter^c$ of $\inter$ by setting $\inter^c(A)=1$ if $\inter(A)$, and $\inter^c(A)=0$ if $\inter(A)<1$, for any $A \in \gatomn$.
Since $\D$ assigns only 1 as a truth degree, $\inter^c$ is a fuzzy model of $\D$.
We claim that $\inter^c$ is also a 1-fuzzy model of $\prog$.
Towards a contradiction suppose that  some
ground rule $\phi \to \phi'$ of $\prog$
is 
1-satisfied by $\inter$, but not by $\inter^c$; that is
$\inter(\phi) \leq \inter(\phi')$ and
$\inter^c(\phi) > \inter^c(\phi')$.
Since
over the Boolean values
any $t$-norm behaves like a conjunction, we obtain that
$\inter^c(\phi) = 1$ and $\inter^c(\phi')=0$.
By the   monotonicity of $t$-norms
$\inter(\phi) \geq \inter^c(\phi)$, so
$\inter(\phi)=1$.
Moreover
$\inter(\phi') <1$ because $\inter(\phi') =1$ would imply that $\inter^c(\phi')=1$. 
Hence  $\inter(\phi) \not\leq \inter(\phi')$, rising a contradiction. Thus
$\inter^c$ is a 1-fuzzy model of $\prog$ and $\D$.
Since $\inter^c$ is a `crisp' model of $\prog$ and $\D$, the fact that $\prog'$ and $\D'$ entail $G$ implies that $\inter^c(G)=1$.
Thus, $\inter(G)=1$, as required.
\end{proof}

\section{Chasing \tdat}
\label{sec:chase}

In this section, we define fuzzy chases and fuzzy universal models.
We will show that  each finite fuzzy chase produces a fuzzy universal model, which can be used to check entailment. 

As described in the previous section, the notion of satisfaction in \tdat is parameterised with  $K \in [0,1]$ (and so are also parameterised notions of a model and entailment),
which determines how much larger the truth degree of a head of rule needs to be than the truth degree of the body so that we treat  the rule as satisfied.
It turns out, however, that the specific choice of the value for $K$ 
does not impact  our technical results.
Therefore, for the sake of simplification, we will assume in the rest of the paper that $K=1$ and we will not mention $K$; for example instead of $K$-\textsc{Entailment} and $\ent{c}{K}$ we will simply refer to \textsc{Entailment} and $\ent{c}{}$, respectively.
In \Cref{discussion}
we will briefly discuss how our results generalise to arbitrary values of $K$.

\subsection{Fuzzy Triggers and Chases}
\label{sec:chasetech}
\label{sec:challenge}
We start by defining notions of 
fuzzy semi-oblivious and restricted chases in \tdat,  
by lifting 
the  notions used in standard \datex  \cite{DBLP:journals/mst/CalauttiP21}.
Our definitions are based on a fuzzy counterpart of a trigger, defined next.

A \emph{fuzzy trigger} 
is a pair $(r, \rho)$, where $r$ is a rule  
and  $\rho$ is a grounding of $r$.
The \emph{result of applying a trigger} $(r, \rho)$ 
to a fuzzy interpretation $\inter$
is a fuzzy interpretation $\inter'$ obtained 
by updating the truth degree of the head of $\rho(r)$  (with existential variables  replaced by nulls)  to the value $\inter(\rho(\body(r)))$, that is, to the truth degree of the body. 
Formally, we let the \emph{head} $\Htrig(r,\rho)$ \emph{of a fuzzy trigger} 
$(r, \rho)$ 
be the ground atom obtained from the head of 
$\rho(r)$ by deleting existential quantifier and replacing each existential  variable $x$ with a null $\mathbf{N}_{r, \rho_{|\fr(r)}}^x \in \nulls$, where the null is
determined by $x$, $r$, and the restriction of $\rho$  to frontier variables $\fr(r)$.
To simplify notation and make it more intuitive, in what follows we will  
write 
$\degr{\inter}{r}{\rho}$ 
instead of $\inter(\rho(\body(r)))$, for the truth degree  corresponding to a trigger application.
Hence, we formally define the result
of applying $(r,\rho)$ to $\inter$ as the following fuzzy  interpretation $\inter'$:
\[
\inter'(\phi) = \begin{cases}
    \degr{\inter}{r}{\rho}, & \text{if } \phi=\Htrig(r,\rho),  \\
    \inter(\phi), & \text{if } \phi \in \gatomn \setminus \{ \Htrig(r,\rho) \}.
\end{cases}
\]
\begin{example}\label{example:trigger}
Let $r$ be the following rule $r$
\begin{align*}
\mathit{NeuralLabel}(x,y) \conjL \mathit{NeuralLabel}(u,w) 
\rightarrow \exists z \mathit{CommonClass}(x,y,z) %
\end{align*}
and let
$\rho$ be its grounding with $x \mapsto \mathit{img1}$, $y \mapsto \mathit{tiger\_shark}$, $u \mapsto \mathit{img2}$, and $w \mapsto \mathit{tench}$, so  the head $\Htrig(r,\rho)$ of the trigger $(r,\rho)$ is  $\mathit{CommonClass(img1,img2, \mathbf{N}_{r, \rho_{|\fr(r)}}^z )} $.
To illustrate application of $(r,\rho)$ assume that $\I$ assigns  $\mathit{NeuralLabel(img1,tiger\_shark)} \mapsto 0.8$ 
and $\mathit{NeuralLabel(img2,tench)} \mapsto 0.9$.
Hence $\degr{\inter}{r}{\rho} = 0.8 \conjL 0.9 = 0.7$, and so, the result of applying $(r,\rho)$ to $\I$ is a fuzzy interpretation $\I'$ with  $\mathit{CommonClass(img1,img2, \mathbf{N}_{r, \rho_{|\fr(r)}}^z )} \mapsto 0.7$.
\end{example}

Depending on the type of allowed triggers we will  obtain different types of chase procedures.
In particular, the semi-oblivious chase will allow for 
\semiob-active triggers and the restrictive chase for \restr-active triggers only.
We say that:
\begin{itemize}
\item a trigger $(r,\rho)$ is \emph{\semiob-active} in a fuzzy interpretation $\inter$ 
if $\degr{\inter}{r}{\rho} > \inter(\Htrig(r,\rho))$, that is, an application of the trigger will increase the current truth degree of $\Htrig(r,\rho)$,

\item 
a trigger $(r,\rho)$ is \emph{\restr-active} in a fuzzy interpretation $\inter$ if
$\degr{\inter}{r}{\rho} > \inter(\Htrig')$
for all 
$\Htrig'$ obtained by replacing nulls in
$\Htrig(r,\rho)$ with elements from $\domn$.
\end{itemize}

\noindent Note that each 
\restr-active trigger is also 
 \semiob-active, because 
$\degr{\inter}{r}{\rho} > \inter(\Htrig')$ for all $\Htrig'$ described above, implies,  in particular, that
 $\degr{\inter}{r}{\rho} > \inter(\Htrig(r,\rho))$.

A \emph{semi-oblivious chase} (\semiob-chase) for a program $\prog$ and a fuzzy dataset $\D$ is a (finite or infinite) sequence of triggers $(r_1,\rho_1),(r_2,\rho_2),\dots$\ with rules $r_i$ in $\prog$ such that 
there exists a sequence of fuzzy interpretations $\inter_0,\inter_1, \inter_2, \dots $ with the following properties:
\begin{enumerate}[label=(\roman*),noitemsep,topsep=3pt]
    \item $\inter_0=\inter_\D$,
\item each $(r_i,\rho_i)$ is $\semiob$-active in $\inter_{i-1}$, 
\item each $\inter_{i}$ (with $i>0$) is the result of applying $(r_i,\rho_i)$ to $\inter_{i-1}$,
    \item for each $\inter_i$ and every \semiob-active trigger in $\inter_i$, 
    there exists finite $j > i$ such that
    this trigger is not \semiob-active in $\inter_j$ (the \emph{fairness} condition).
    \label{cond.fairness}
\end{enumerate}
Observe that each fuzzy chase has a unique sequence $\inter_0,\inter_1, \inter_2, \dots $ of fuzzy interpretations, which we will call \emph{corresponding} %
 to the chase;
we will also call $\inter_{i}$ the  \emph{$i$th interpretation in the chase}.
A \emph{restricted chase}  (\restr-chase) is defined analogously, by using the notion of a
 \restr-active trigger instead of an \semiob-active trigger.
Note that both types of chases can be either finite or infinite.
If a chase is finite, we call the last interpretation $\I_n$ the \emph{result of applying the chase}.

\begin{example}\label{example:chase}
Let $\prog = \{ r\}$ for the rule $r$ from \Cref{example:trigger}  and
let $\D$ be the fuzzy dataset corresponding to $\I$ from \Cref{example:trigger} (i.e., $\I_\D = \I$).
There are four active triggers in $\I_\D$ (as there are four ways we can assign $\mathit{img1}$ and $\mathit{img2}$ to $x$ and $y$).
Each order of their applications gives rise to a different chase. For example we can obtain a chase, say $(r,\rho_1), (r,\rho_2), (r,\rho_3), (r,\rho_4)$, with corresponding interpretations $\I_0,  \I_1, \I_2, \I_3, \I_4$ such that
\begin{align*}
& \I_1 \text{ assigns } \mathit{CommonClass(img1,img2, \mathbf{N}_{r, \rho_1{_{|\fr(r)}}}^z )} \mapsto 0.7 \qquad (\text{since } 0.8 \conjL 0.9 =0.7),
\\
& \I_2 \text{ assigns } \mathit{CommonClass(img2,img1, \mathbf{N}_{r, \rho_2{_{|\fr(r)}}}^z )} \mapsto 0.7 \qquad (\text{since } 0.9 \conjL 0.8 =0.7),
\\
& \I_3 \text{ assigns } \mathit{CommonClass(img1,img1, \mathbf{N}_{r, \rho_3{_{|\fr(r)}}}^z )} \mapsto 0.6 \qquad (\text{since } 0.8 \conjL 0.8 =0.6),
\\
& \I_4 \text{ assigns } \mathit{CommonClass(img2,img2, \mathbf{N}_{r, \rho_4{_{|\fr(r)}}}^z )}\mapsto 0.8 \qquad (\text{since } 0.9 \conjL 0.9 =0.8).
\end{align*}
Note that $(r,\rho_1), (r,\rho_2), (r,\rho_3), (r,\rho_4)$ is both an \semiob-chase and an \restr-chase for $\prog$ and $\D$.
It is still an \semiob-chase if we add $\mathit{CommonClass(img1,img2,fish)} \mapsto 0.8$ to $\D$, but it is not an \restr-chase any more because
$\degr{\inter_0}{r}{\rho_1} = 0.7 \not> 0.8 =
\inter_0(\mathit{CommonClass(img1,img2,fish)})$, and so,
$(r,\rho_1)$ is not \restr-active in $\I_0$.
\end{example}

\begin{comment}
\begin{proposition}\todo{M: ideally find somewhere where someone wrote out the standard argument and just cite it. Otherwise maybe just write a sketch. It's painfully boring and a waste of space to argue this in full.}
    Let $\Pi$ be a $t$-Datalog$^\exists$ program and let $\tau$ be a fuzzy dataset:
    \begin{itemize}
        \item If there is a failing semi-oblivious chase sequence for $\Pi,\tau$, then all semi-oblivious chase sequences for $\Pi,\tau$ fail.
        \item If there is a failing restricted chase sequence for $\Pi,\tau$, then all restricted chase sequences for $\Pi,\tau$ fail.
    \end{itemize}
\end{proposition}
\end{comment}

\subsection{Fuzzy Universal Models}\label{sec:universal}

The core property of  chases  in \datex, which makes them a crucial tool for checking entailment, is that each finite chase results in a \emph{universal model} that represents (modulo homomorphisms) all models of a given program and dataset. 
In this section, we show that an analogous result can be provided for fuzzy chases in \tdat.
We start by introducing a notion of a homomorphism which is tailored  to our fuzzy setting.
As defined below, a homomorphism is a function 
$h : \domn \longrightarrow \domn$, but we will often use its extension to atoms and rule bodies;
for example we will write $h(A)$ to refer to the atom $A$ with constants replaced according to $h$. 

\begin{definition}
\label{def:ndfh}
    A \emph{non-decreasing homomorphism} from  a fuzzy interpretation  $\inter$ to  a fuzzy interpretation $\inter'$ is any function $h : \domn \longrightarrow \domn$ such that:
\begin{enumerate}[label=(\roman*),noitemsep,topsep=3pt]
    \item $h(a) = a$, for every $a\in \dom$, and
    \label{univmodel.i}
    \item  $\inter(A) \leq \inter'(h(A))$, for every $A \in \gatomn$.
\end{enumerate}
We will write $\inter \ndfh \inter'$ if there is a non-decreasing  homomorphism from $\inter$ to $\inter'$.
\end{definition}

Similarly as in \datex, we define a  universal model of $\prog$ and $\D$ as  a model which can be homomorphically mapped to any model of $\prog$ and $\D$.
However, instead of the standard homomorphism, we use the above-defined non-decreasing homomorphism.

\begin{definition}
A \emph{fuzzy universal model} of a \tdat program $\prog$ and a fuzzy dataset $\D$ is any fuzzy model $\I$ of $\prog$ and $\D$ such that $\inter \ndfh \inter'$ for every fuzzy model  $\inter'$ of $\prog$ and $\D$.
\end{definition}

In general, a fuzzy
universal model may not be unique.
However, if a program does not have existential variables,  we can show that there exists a unique fuzzy universal model.
Note that this is analogous to the standard, non-fuzzy setting.

\begin{theorem}
\label{minimalmodel}
Each pair of a \tda program $\prog$ 
and a  fuzzy dataset $\D$ has a unique
fuzzy universal model.
\end{theorem}
\begin{proof}
Consider the set $S$
of all fuzzy models of $\prog$ and $\D$.
We claim that the fuzzy interpretation $\I$ with $\I(A) = \inf \{ \I'(A) \mid \I' \in S \} $, for each $A \in \gatomn$, is the unique fuzzy universal model of $\prog$ and $\D$.
By the construction, $\I$ is a fuzzy model of $\D$.
To show that it is a fuzzy model of $\prog$, we observe that for any rule $\phi \to \phi'$ in $\prog$, any of its  groundings $\rho$, and any $\I' \in S$, we have
 $\I'(\rho(\phi)) \leq \I'(\rho(\phi'))$.
Thus, we have also $\I(\rho(\phi)) \leq \I(\rho(\phi'))$.
Consequently, $\I$ is a fuzzy model of $\prog$ and $\D$ (i.e, $\I \in S$).
To prove that $\I$ is a fuzzy universal model, we need to show that $\I \ndfh \I'$, for any $\I' \in S$.
This, however, is witnessed by simply letting $h$ be the identity function on $\domn$.
Indeed, such $h$ satisfies both Conditions (i) and (ii) from \Cref{def:ndfh}.
Finally, we observe that, by the construction, no  $\I' \in S$ satisfies 
$\I' \ndfh \I$, so no $\I' \in S$
different from $\I$ can be a fuzzy universal model of $\prog$ and $\D$.
\end{proof}

We observe that each fuzzy universal model can be used to check entailment, which follows directly from our definitions.

\begin{proposition}
\label{univ.entailment}
    Let $\inter$ be a fuzzy universal model of a \tdat program $\prog$ and a fuzzy dataset $\D$, let $c \in [0,1]$, and let $G \in \gatom$. 
    Then  $\inter(G) \geq c$ if and only if $(\Pi,\D) \ent{c}{} G$.
\end{proposition}

The crucial property connecting fuzzy chases and universal models, is that every finite fuzzy chase (semi-oblivious or restricted)
needs to result in a fuzzy universal model.

\begin{theorem}
\label{universalmodel}
Let $\star\in\{\semiob,\restr\}$.
The result of a (finite) $\star$-chase for a \tdat program $\prog$ and a fuzzy dataset $\D$ is a fuzzy universal model of $\prog$ and $\D$.
\end{theorem}
\begin{proof}%
Consider first  $\star=\semiob$.
Let 
$(r_1,\rho_1),(r_2,\rho_2),\dots , (r_n,\rho_n)$ be a finite $\semiob$-chase for $\prog$ and $\D$,
and  let  $\I_\D=\I_0,\I_1, \dots \I_n$  be corresponding  fuzzy interpretations.
To show that $\I_n$ is a fuzzy universal model for $\prog$ and $\D$,
we fix an arbitrary fuzzy model $\I'$ of $\prog$ and $\D$;
we will show inductively on $i$ that $\I_i \ndfh \inter'$.

For the basis of induction, we need to show that there is a non-decreasing fuzzy homomorphism from $\I_\D$ to $\I'$.
Since $\I'$ is a fuzzy model of $\D$, we have  $\I_\D(A) \leq \I'(A)$ for every $A \in \gatomn$.  Thus the identity function on \domn is a
non-decreasing fuzzy homomorphism  witnessing  $\I_\D \ndfh \I'$.

For the inductive step assume that $h$ is a 
non-decreasing fuzzy homomorphism from $\I_{i-1}$ to $\I'$; we will show how to construct a  non-decreasing fuzzy homomorphism $h'$ from $\I_{i}$ to $\I'$.
Let us write $\mathit{body}$ as a shorthand for $\rho_i(\body(r_i))$.
Since $h$ is a  
non-decreasing fuzzy homomorphism from $\I_{i-1}$ to $\I'$ and $t$-norms are monotone, $\I_{i-1}(\mathit{body}) \leq \I'( h(\mathit{body}) )$.
As $\I'$ is a fuzzy model of $\prog$, it needs to satisfy $r_i$.
Thus, there exists a ground atom $H$, obtained by replacing each existential variable $x$ in $\head(r_i)$ with some $a_x \in \domn$,
such that $\I'(h(\mathit{body})) \leq \I'(H)$.
We use this $H$ to define $h'$ as follows:
\[
h'(a) = \begin{cases}
   a_x, & \text{if } a = \mathbf{N}_{r_i, \rho_{i|\fr(r_i)}}^x \text{ for some } x \text{ (i.e., $a$ is a null in $\Htrig(r_i, \rho_i)$)},  \\
    h(a), & \text{for all other  $a\in \domn$}.
\end{cases}
\]
It remains to show that $h'$ is  a non-decreasing fuzzy homomorphism from $\I_{i}$ to $\I'$. 
Condition~(i) of \Cref{def:ndfh} holds by the construction of $h'$. 
For Condition~(ii) it suffices to 
show that $\I_{i}(\Htrig(r_i,\rho_i)) \leq \I'(h'(\Htrig(r_i,\rho_i)))$, as for all  $A \in \gatomn$ other than $\Htrig(r_i,\rho_i)$, the inequality holds by the  inductive assumption.
We observe that the following hold:
    \[
    \I_{i}(\Htrig(r_i,\rho_i)) = \I_{i-1}( \mathit{body} ) \leq \I'(h(\mathit{body})) \leq \I'(H) = \I'(h'(\Htrig(r_i,\rho_i))).
    \]
The first equality holds by the definition of a trigger application, 
the next two inequalities are  already shown, 
and the last equality holds by the construction of $h'$.
Thus $h'$  witnesses $\I_i \ndfh \inter'$.

Finally, we observe that each finite \restr-chase is a prefix of some \semiob-chase, so the  argumentation above also proves the theorem for $\star = \restr$.
\end{proof}

\Cref{universalmodel}, together with \Cref{univ.entailment},  provides us with a mechanism for checking entailment, which aligns directly with the standard methods for deciding entailment in \datex. 
There are however two main difficulties that need to be addressed.
First, as in the classical setting, fuzzy chases are not always finite.
Second, unlike in the classical setting, a fuzzy chase can update \emph{multiple times} the truth degree of the same 
ground atom, as
illustrated in the following example.

\begin{example}
Consider a program $\prog$  with Rules \eqref{ex1}--\eqref{ex3}, 
and a  fuzzy dataset $\D$ assigning
    $\mathit{NeuralLabel}(img,c_1) \mapsto 0.9,  \mathit{Class}(img,c_1) \mapsto 0.6$, and $\mathit{Hypernym}(c_1,c_2) \mapsto 1$.
Let the first trigger in a  fuzzy chase for $\prog$ and $\D$  be $(r_2, \{x \mapsto img, y \mapsto c_1, z \mapsto c_2\})$, which results in $\mathit{Class}(img,c_2) \mapsto 0.6$. 
Let the second trigger be
$(r_1, \{x\mapsto img, y \mapsto c_1\})$, which assigns $\mathit{Class}(img,c_1) \mapsto 0.9$. 
This, however, makes $(r_2, \{x\mapsto img, y \mapsto c_1, z \mapsto c_2\})$
active again and its application now yields $\mathit{Class}(img,c_2) \mapsto 0.9$.
Note that it holds in both  \semiob- and \restr-chases.
\end{example}

The example above shows that even in the case of a \tda  program (with no existential quantification), a trigger can be \emph{reactivated}, and so, the same trigger can occur multiple times in a fuzzy chase.
This, in turn, may potentially lead to exponentially long or even infinite chases, which never happens in standard Datalog.
Clearly, if existential quantification is present in a \tdat program,  reasoning with chases becomes even more challenging.   
In the next section, we will show how to overcome these difficulties.
In particular, we will show in which cases the finiteness of a chase is guaranteed, and so, chasing can be used as a decision procedure for entailment checking.

\section{Truth-Greedy Chases}
\label{sec:termination}

As we have shown, reasoning is conceptually harder than in the classical setting, as a fuzzy chase may modify multiple times a truth degree of the same atom.
To address this difficulty, we will introduce truth-greedy chases, which allow us to overcome the above-mentioned issue.
This, in particular, will allow us to show that reasoning in weakly acyclic programs in \tdat is tractable for data complexity, and so, no harder than in \datex.

Let $\star \in \{ \semiob, \restr \}$.
We say that a $\star$-chase $(r_1,\rho_1),(r_2,\rho_2),\dots$  with corresponding fuzzy interpretations $\I_0, \I_1, \I_2, \dots$ is 
\emph{truth-greedy} 
if
each of its triggers $(r_i,\rho_i)$ has a maximal truth degree among all 
$\star$-active triggers in $\I_{i-1}$,
that is,
there is no $\star$-active trigger  $(r_i,\rho_i)$ such that 
$\degr{\I}{r}{\rho} < \degr{\I}{r'}{\rho'}$.
In other words, in the $i$th step of a truth-greedy chase, we need to apply a trigger whose results assigns a maximal truth degree among active triggers.

As we show next,  truth degrees of triggers in a truth-greedy chase are non-increasing.
Note that this does not follow directly from the definition of a truth-greedy trigger, because
application of a trigger can activate triggers which were not active before.
In general, such triggers may have  higher truth degrees.
However, we will show that due to monotonicity of $t$-norms none of the newly activated triggers can have a truth degree higher than previously applied triggers.

\begin{proposition}
\label{lem:monotone.triggers}
Let $\star \in \{\semiob,\restr\}$ and let $(r_1,\rho_1),(r_2,\rho_2),\dots$ be a truth-greedy $\star$-chase
with corresponding  interpretations  
$\I_0,\I_1, \I_2, \dots$.  
For all $i \geq 1$ it holds that  $\degr{\I_{i-1}}{r_i}{\rho_i} \geq \degr{\I_{i}}{r_{i+1}}{\rho_{i+1}}$.
\end{proposition}
\begin{proof}
Suppose towards a contradiction that 
$\degr{\I_{i-1}}{r_i}{\rho_i} < \degr{\I_{i}}{r_{i+1}}{\rho_{i+1}}$, for some $i \geq 1$.
Since the chase is truth-greedy, 
$\degr{\I_{i-1}}{r_{i+1}}{\rho_{i+1}} \leq \degr{\I_{i-1}}{r_{i}}{\rho_{i}}$. 
Therefore 
$\degr{\I_{i-1}}{r_{i+1}}{\rho_{i+1}}  < \degr{\I_{i}}{r_{i+1}}{\rho_{i+1}}$,
or equivalently
$\I_{i-1} ( \rho_{i+1}(\body(r_{i+1})) ) <  \I_{i} ( \rho_{i+1}(\body(r_{i+1})) )$.
By the definition of a trigger application, the only difference between $\I_{i-1}$ and $\I_i$ is that the truth degree of $\Htrig(r_{i}, \rho_{i})$ is strictly increased in $\I_i$ to the value of 
$\I_{i-1} ( \rho_{i}(\body(r_{i})) )$.
Therefore, by the fact that $\I_{i-1} ( \rho_{i+1}(\body(r_{i+1})) ) <  \I_{i} ( \rho_{i+1}(\body(r_{i+1})) )$, the body $\rho_{i+1}(\body(r_{i+1}))$ needs to
mention the atom $\Htrig(r_{i}, \rho_{i})$,
and so, by \Cref{minislargest},
the truth degree of this body in $\I_i$ is no larger than the truth degree of $\Htrig(r_{i}, \rho_{i})$, that is,
$\I_{i} ( \rho_{i+1}(\body(r_{i+1})) ) \leq \I_{i-1} ( \rho_{i}(\body(r_{i})) )$.
This, however, contradicts 
the assumption $\degr{\I_{i-1}}{r_i}{\rho_i} < \degr{\I_{i}}{r_{i+1}}{\rho_{i+1}}$.
\end{proof}

As a consequence of \Cref{lem:monotone.triggers} and
 the definitions of \semiob- and \restr-active triggers, we obtain that 
a truth-greedy chase 
cannot mention two triggers with the same head. 

\begin{corollary}\label{norep}
    Let $\star \in \{\semiob,\restr\}$ and  let $(r_1,\rho_1),(r_2,\rho_2),\dots$ be a truth-greedy $\star$-chase. 
    It holds that $\Htrig(r_i,\rho_i) \neq \Htrig(r_j,\rho_j)$ whenever
    $i \neq j$.
\end{corollary}

\begin{comment}

\begin{lemma}
Let $\star \in \{\semiob,\restr\}$.
    Let $\Pi,\D$ be a \tdat program and fuzzy database. Suppose that every truth-greedy $\star$-chase is finite. Then every $\star$-chase of $\Pi,\D$ is finite.
\end{lemma}
\begin{proof}
    {\color{red}
    Something about triggering stuff until reaching the next truth-greedy step}
\end{proof}
%
    
\end{comment}

To exploit truth-greedy chases for efficient reasoning,
we will relate fuzzy chases in \tdat to standard chases in \datex.
In formal terms, we will relate a fuzzy chase of a \tdat program $\Pi$ and a fuzzy dataset $\D$ to a fuzzy chase of  $\Pi$ and $\D^{\crisp}$, where 
$\D^{\crisp}$ is a `crispified' version of $\D$ obtained by setting $\D^{\crisp}(A)=1$ whenever $\D^{\crisp}(A)>0$.
Note that a fuzzy chase for a `crispified'  dataset, corresponds to a standard chase in \datex.
As we show next, if the application of standard chase procedures to a crispified dataset $\D^\crisp$ always terminates (e.g., when a program has no existential variables, it is non-recursive, or weakly-acyclic), then for each
fuzzy chase $s$ there is a 
standard chase $s^\crisp$ such that  $s$  assigns truth degrees to no more atoms than $s^\crisp$.
Formally, for a chase $s$ with corresponding interpretations $\I_0, \I_1, \dots$ we let $\atoms(s)$ be the set of all $A \in \gatomn$ such that $\I_i(A) \neq 0$, for some $i$.
Hence, our result claims that  
$\atoms(s)\subseteq \atoms(s^\crisp)$, shown  below.

\begin{lemma}
\label{chase.embed}
    Let $\star \in \{\semiob,\restr\}$, let $\prog$ be a \tdat program, and let $\D$ be a fuzzy dataset. Assume that every sequence of $\star$-active trigger applications to $\prog$ and $\D^\crisp$ is finite.
    Then, for every $\star$-chase $s$ of $\prog$ and $\D$ there exists a $\star$-chase $s^\crisp$ of $\prog$ and $\D^\crisp$ with $\atoms(s)\subseteq \atoms(s^\crisp)$.    
\end{lemma}
\begin{proof}
Let $s$ be a $\star$-chase $(r_1,\rho_1), (r_2,\rho_2), \dots$ with corresponding  interpretations $\I_0, \I_1, \I_2, \dots$.
We will show inductively on $i \geq 1$, that for each prefix $s_i=(r_1,\rho_1), \dots, (r_i,\rho_i)$ of $s$,
there exists a finite 
sequence 
$s_i' = (r_1',\rho_1'),  \dots, (r_j',\rho_j')$ of $j \leq i$ $\star$-active trigger applications to $\prog$ and $\D^\crisp$ such that $\atoms(s_i) = \atoms(s_i')$.
By the assumption that every sequence of $\star$-active trigger applications to $\prog$ and $\D^\crisp$ is finite, this implies the existence of the required $s^\crisp$.

In the basis of the induction 
we have $s_1 = (r_1,\rho_1)$, hence 
$(r_1,\rho_1)$ is $\star$-active in $\I_\D$, and so,
$\rho_1(\body(r_1)) > 0$.
Thus, for each 
$A$ mentioned in $\rho_1(\body(r_1))$ we have
$\D(A)>0$ and therefore $\D^\crisp(A)=1$.
If $\D^\crisp (\Htrig(r_1,\rho_1) ) \neq 1$,  then $(r_1,\rho_1)$ is $\star$-active in $\I_{\D^\crisp}$ and we let $s_1' = (r_1,\rho_1)$; otherwise we let $s_1'$ be empty sequence.
In both cases $\atoms(s_1) = \atoms(s_1')$, as required.

In the inductive step we fix $i \in \mathbb{N}$ and assume that the claim is witnessed for $s_i=(r_1,\rho_1), \dots, (r_i,\rho_i)$ by $s_i'=(r_1',\rho_1'),  \dots, (r_j',\rho_j')$.
Since $(r_{i+1},\rho_{i+1})$ is $\star$-active in $\I_i$, for each atom $A$ in the body $\rho_{i+1}(\body(r_{i+1}))$, we have $\I_i(A) >0$, and so, $A \in \atoms(s_i)$.
Thus, by the inductive assumption, 
$A \in \atoms(s_i')$, that is, $\I_j'(A)=1$ for $\I_0', \dots ,\I_j'$ being interpretations corresponding to $s_i'$. 
Similarly as in the basis, if 
$\I_j'(\Htrig(r_{i+1},\rho_{i+1}) ) \neq 1$,  then $(r_{i+1},\rho_{i+1})$ is $\star$-active in $\I_j'$ and we let $s_{i+1}' = (r_1',\rho_1'),  \dots, (r_j',\rho_j'), (r_{i+1}, \rho_{i+1})$; otherwise we let $s_{i+1}'=s_{i}'$.
In both cases $\atoms(s_{i+1}) = \atoms(s_{i+1}')$.
\end{proof}

It is worth observing that checking
the assumption of \Cref{chase.embed}, that is, if all $\star$-trigger applications to $\prog$ and $\D^\crisp$ are finite,  we can use results established for \datex{}.
For example, it is known that for a fixed program and a dataset, all semi-oblivious chases are finite if there exists some finite semi-oblivious chase~\cite{DBLP:journals/fuin/GrahneO18}. 
Since  semi-oblivious chases in \datex coincide with  $\semiob$-chases in \tdat for $\prog$ and $\D^\crisp$, checking if the assumption of
\Cref{chase.embed} is satisfied, reduces to checking if some $\semiob$-chase of $\prog$ and $\D^\crisp$ is finite.

Note that \Cref{chase.embed} and \Cref{norep} 
allow us to determine for which pairs of a program $\prog$ and a dataset $\D$, the truth-greedy chase is guaranteed to terminate.
Indeed, this is the case whenever the standard chase applied to $\D^\crisp$ is guaranteed to terminate.
Furthermore, if we know the lengths of standard chases, we can bound the length of fuzzy truth-greedy chases.

For example, let us consider \emph{weakly acyclic} programs (a syntactic restriction that inhibits the role of nulls in recursion~\cite{dataexchange}),
which constitutes one of the most prominent fragments of \datex
with terminating 
restricted chase~\cite{dataexchange}.
Since standard chases for weakly acyclic programs are known to be polynomially long in the size of a dataset, we can show that truth-greedy chases are also polynomially long.
Moreover, we can show that entailment in weakly acyclic \tdat programs is \ptime-complete for data complexity, that is, of the same computational complexity as in the non-fuzzy weakly acyclic \datex programs.

\begin{theorem} 
\textsc{Entailment} for weakly acyclic \tdat programs is \ptime-complete in data complexity.
\end{theorem}
\begin{proof}%
\begin{comment}\todo{I would write the whole proof here instead of mentioning the log-space results, which seem to be  straight-forward}
\end{comment}
The lower bound follows from  \ptime-hardness of entailment in weakly-acyclic \datex~\cite{DBLP:conf/rr/CaliGP10} and \Cref{thm:all1}.
For the upper bound assume that we want to check if $\PD \ent{c}{} G$, for a weakly acyclic \tdat program $\prog$, fuzzy dataset $\D$, $G \in \gatomn$, and  $c \in [0,1]$.
By \Cref{univ.entailment} and \Cref{universalmodel}
it suffices to construct a (finite) \restr-chase  for $\prog$ and $\D$, and to check if $\I(G) \geq c$ in the resulting fuzzy interpretation $\I$ of this chase.
Since $\prog$ is weakly-acyclic, by the result for  \datex established by Fagin~et~al.~\cite{dataexchange}, every sequence of $\restr$-active trigger applications to $\prog$ and $\D^\crisp$ is of polynomial length in the size of $\D^\crisp$.
Hence, by \Cref{chase.embed},
$|\atoms(s)|$ is polynomial, for every \restr-chase $s$ of $\prog$ and $\D$. 
Note that this does not mean that all \restr-chases have polynomial lengths. 
However, by \Cref{norep}, we obtain that all truth-greedy \restr-chases have polynomial lengths.

It remains to argue that constructing a truth-greedy \restr-chase of $\prog$ and $\D$ is feasible in polynomial time (in the size of $\D$).
To construct the chase, in every step we compute all \restr-active triggers $(r,\rho)$,  choose one of the triggers with maximal truth degree of $\rho(\body(r))$, and apply it to construct a next fuzzy interpretation.
This can be done in logarithmic space because there are polynomially many triggers to consider (in particular, there are polynomially many groundings $\rho$ to consider, as they assign to variables only those constants, for which some atom in the current interpretation has a non-zero truth degree), so
we can inspect all of them keeping in memory only a single trigger with the highest truth degree of $\rho(\body(r))$.
Applying the computed trigger is also feasible in logarithmic space.
Since a truth-greedy chase has a polynomial length, we can compute the final interpretation
and check if 
the truth degree of $G$ in this  interpretation is at least $c$, in polynomial time.
\end{proof}

Since each program with no existential variables is weakly acyclic, we obtain as a corollary that entailment in \tda (\tdat with no existential variables) is also in \ptime for data complexity.
Matching lower bound follows from \Cref{thm:all1} and 
\ptime-hardness of entailment in Datalog for data complexity.

\begin{corollary}
\label{tDatalog.ptime}
    \textsc{Entailment} for  \tda programs is \ptime-complete in data complexity.
\end{corollary}

\section{Adding Negation and Other Unary Operators}\label{sec:negation}

In this section, we will show how to extend our fuzzy setting with unary operators (also studied under the name of \emph{negators} \cite{fodor2000fuzzy}) without negative impact on the complexity of reasoning.
Unary operators are interpreted in our fuzzy setting as functions mapping a 
truth degree of a ground atom into another truth degree.
A flagship unary operator is negation, which in the fuzzy setting can be interpreted in various ways.
However, 
there are also many other unary operators worth considering. For instance, threshold operators $\Delta_T$, for $T \in [0,1]$ which assign truth degree 1 to atoms which have truth degree at least $T$.
Below we present semantics of two versions of fuzzy negation and of the threshold operator:
\begin{align*}
\I(\neg A) = 1 - \I(A),
&&
\I(\sim\!A) =
\begin{cases}
    1, & \text{if } \I(A)=0, \\
    0, & \text{otherwise},
\end{cases}
&&
\I(\Delta_T A) = \begin{cases}
    1, & \text{if } \I(A) \geq T, \\
    0, & \text{otherwise}.
\end{cases}
\end{align*}

\noindent Formally, we let
a unary operator be any computable function $U: [0,1] \longrightarrow [0,1]$. 
A \tdatu program is defined similarly as a \tda program, but body atoms in rules can be preceded by arbitrary unary operators.\footnote{Similarly as in the case of $t$-norms we slightly abuse notation by using the same symbol for a unary operator and a function interpreting this operator.} Hence, rules of a \tdatu program are of the form
\begin{align*}
    U_1 R_1(\mathbf{x_1}) \conj_r \cdots \conj_r U_\ell R_\ell(\mathbf{x_\ell}) & \rightarrow  S(\mathbf{x_h}, \mathbf{z}),
\end{align*}
where each $U_i$ is a unary operator or is empty.
As usual in Datalog with negation~\cite{DBLP:books/aw/AbiteboulHV95}, we define a notion of a stratified program as follows.
We let 
a \emph{stratification} of a \tdatu program $\Pi$ be
any function
$\sigma$ mapping predicates mentioned in $\prog$ to positive integers
such that  for each rule $r \in \prog$ and all
 predicates $P$, $P^+$, and $P^-$ mentioned, respectively, in the head, in body atoms not using unary operators, and in body atoms using unary operators of $r$, 
it holds that 
$\sigma(P^+) \leq \sigma(P)$ 
and $\sigma(P^-) < \sigma(P)$.
We will  treat a stratification $\sigma$ as a partition of $\prog$ into 
$\Pi_1,\dots,\Pi_n$ 
such that $n$ is the maximal value assigned by $\sigma$ and $\prog_i$ consists of all rules in $\prog$ with heads $P$ such that $\sigma(P)=i$.
A program is \emph{stratifiable} (or stratified) if it has some stratification.

Semantics of unary operators is straight forward, namely 
$\I( U A) = U(\I(A))$ for each fuzzy interpretation $\I$ and  unary operator $U$.
The definition of entailment, however, is more complex;
we will adapt the procedural semantics used in the non-fuzzy setting~\cite{DBLP:books/aw/AbiteboulHV95}.
To this end, we start by considering \emph{semi-positive} \tdatu programs, in which  unary operators appear only in front of extensional predicates (i.e., predicates which do not occur  rule heads).
We can use the argumentation from the proof of \Cref{minimalmodel} to show that semi-positive \tdatu programs, similarly to \tda programs,  have the unique fuzzy model property. 
\begin{proposition}\label{uniquesemi}
Each pair of a semi-positive \tdatu program $\prog$ 
and a  fuzzy dataset $\D$ has a unique
fuzzy universal model.
\end{proposition}

\noindent 
Next, we exploit  \Cref{uniquesemi}, to define semantics for stratifiable programs.
Given a stratification $\prog_1, \dots, \prog_n$ of $\prog$ and a fuzzy dataset $\D$ we let 
$\I_0,\I_1, \dots, \I_n$ be a sequence of fuzzy interpretations such that $\I_0 = \I_\D$ and each $\I_{i}$ with $i>1$ is the  fuzzy universal model of $\prog_i$ and (the dataset representation of) $\D_i$.
By the definition,  each $\prog_i$ is semi-positive, so $\I_{i}$ is well defined by \Cref{uniquesemi}.
We call $\I_n$ the result of applying $\prog_1, \dots, \prog_n$ to $\D$.
Similarly, as in the non-fuzzy setting, we can show that each stratification results in the same interpretation.

\begin{proposition}\label{allstrat}
Let $\prog$ be a stratifiable \tdatu program and $\D$ a fuzzy dataset.
Applying any stratification of $\prog$ to $\D$ results in the same fuzzy interpretation.
\end{proposition}

\noindent Establishing \Cref{allstrat} allows us to introduce a standard  definition of entailment in stratified programs.
We say that an atom $G$ is $c$-\emph{entailed}, for any $c \in [0,1]$, by a stratifiable $\prog$ and $\D$ if
$\I_n(G) \geq c$, where $\I$ is the result of applying any  stratification of $\prog$ to $\D$.
We define \textsc{Entailment} problem analogously as in the case of \tdat, namely, given a 
\tdatu program $\prog$, a fuzzy dataset $\D$,  $G \in \gatom$, and $c \in[0,1]$, the problem is to check if 
$\prog$ and $\D$ $c$-entail $G$.

The main result of this section is that adding arbitrary unary operators to \tda does not increase the complexity of entailment checking, as long as the input program is stratifiable.
In other words, entailment checking for stratifiable \tdatu programs is no harder than entailment checking for \tda programs.
As in the case of $t$-norms, our complexity analysis does not take into account the complexity of computing application of unary operators.
Although such computations are usually of low  complexity (e.g., in the case of $\neg$, $\sim$, and $\Delta_T$ introduced at the beginning of this section),
but one can introduce  much more complex unary operators.

\begin{theorem}\label{strat:complex}
    \textsc{Entailment} for stratifiable \tdatu programs is \ptime-complete in data complexity.
\end{theorem}
\begin{proof}[Proof sketch]
The lower bound is inherited from Datalog.
For the upper bound, we compute any stratification $\prog_1, \dots, \prog_n$ of the input program.
Then we compute minimal fuzzy models $\I_1, \dots, \I_n$ corresponding to application of the stratification.
To compute each $\I_{i+1}$ from $\I_{i}$ and $\prog_{i+1}$
we apply
the truth-greedy chase.
By our results on truth-greedy chases in \Cref{sec:termination},  this procedure allows us to construct $\I_n$  in polynomial time with respect to $\D$.
Finally, checking if $\I_n(G) \geq c$, for input $G$ and $c$, is clearly  feasible in polynomial time.
\end{proof}

%
%
%

\begin{comment}
In particular, if we treat disregard  complexity of computing unary operators---as we did also in the case of t-norms---we obtain the following result.

\begin{theorem}
\label{thm:stratified}
    \textsc{$1$-Entailment} is \ptime-complete in data-complexity for weakly acyclic \tdatu programs and the lower bound holds already for programs with no existential quantification.
\end{theorem}
\end{comment}

\begin{comment}
 
\section{Discussion}

\todo[inline]{P:
Mention that we can mix t-norms but assuming  left to right asaociativity
}

\todo[inline]{P:
mention somewhere in the end that all our results generalise for all $K$.

In Section 3.2:

\emph{$K$-fuzzy universal model} of $\Pi,\D$ if it is a $K$-fuzzy model of $\Pi,\D$, and for every $K$-fuzzy model $\I'$ of $\Pi,\D$ there exists a non-decreasing fuzzy homomorphism from $\I$ to $\I'$.

And all the results shown there hold also for any $K$.

Also mention that allowing for constraints (rules with $\bot$ in head) is "free". This we can argue informally since their behaviour in chase procedures is well established and our framework extends that naturally. (easiest is to just see $\bot$ as a special propositional variable, by fairness it will trigger in all considered cases and we know that it's inconsistent).
}
   
\end{comment}

\section{Discussion and Future Work}
\label{discussion}
\paragraph{Summary.}
We have introduced \tdat, an extension of \datex  for reasoning about uncertain information,  which can be used for neuro-symbolic reasoning.
By allowing for arbitrary $t$-norms in the place of conjunctions in standard Datalog programs,  \tdat provides us with a highly  flexible fuzzy formalism.
We have established a chase-based reasoning technique applicable to the fuzzy setting of \tdat, which we used for computational complexity analysis. 
For example, the complexity of entailment in weakly-acyclic \tdat matches that of \datex, namely it is \ptime-complete for data complexity.
Our fuzzy chase procedure is worst-case optimal  for such reasoning.
Moreover, we  showed that \tda can be extended with arbitrary (fuzzy) unary operators without a negative impact on the complexity.
Our development has purposefully followed that of \datex in order to leverage the wide range of results known in the classical case. 
The obtained results illustrate the advantages of this approach and lay the foundation to
richer fuzzy ontology languages in which reasoning can be performed efficiently, akin to the Datalog$^\pm$ family of languages~\cite{gottlob2014datalog}.

\paragraph{Discussion.}
It is worth discussing some generalisations of the presented work.
Recall that to simplify the presentation, we have focused on the proofs on $K$-satisfiability with  $K=1$. 
However, our techniques are not specific to this case, and our results can be extended to  arbitrary  $K\in[0,1]$. 
Furthermore, observe that the presented syntax of  \tdat requires that each rule mentions at most one $t$-norm.
For example, we do not allow for a rule of the form $P(x) \conjL Q(x) \conj_{\product}  R(x) \to S(x)$, which mentions two different $t$-norms.
We have imposed such a restriction in order to obtain associativity of operators in rule bodies, which is the case in standard Datalog and logic programming.
However, if we fix the order of $t$-norms application in rule bodies, we can use multiple 
$t$-norms within a single rule and it would not impact our technical results.
Furthermore, we have only explicitly considered atomic fact entailment and data complexity measure in the paper, but
our techniques are also applicable to more complex queries and allow for performing analysis of other measures of computational complexity (e.g., combined or program complexity).

\paragraph{Future Work.}
Our research on \tdat is motivated by the  need for practical and  powerful reasoning formalism for heterogeneous data sources. We have intentionally chosen to align the development of \tdat  closely  with \datex,
to exploit results known for the latter.
Another great benefit of this approach  is that it opens a way for implementing 
reasoning procedures for 
\tdat by extending existing \datex reasoning systems. 
In particular, developing implementations on top of established chase-based reasoners, such as Vadalog~\cite{DBLP:journals/is/BellomariniBGS22}, form an attractive opportunity for future work.

The  results obtained in the paper also introduce interesting theoretical research directions.
Notice that in the presented results  we focus  on programs with finite \datex chase which, as we have shown, allows us to  construct finite chases in \tdat.
However, there are prominent  fragments of \datex, such as (\emph{weakly}) \emph{guarded}~\cite{DBLP:journals/jair/CaliGK13} or \emph{warded}~\cite{DBLP:journals/tods/BergerGPS22} programs, where entailment is decidable even though the chase is generally not finite. 
Our truth-greedy chase technique does not provide us with an appropriate tool for reasoning in such cases.
For instance, consider the behaviour of  a truth-greedy chase applied to a guarded program as described below.

\begin{example}
Consider a fuzzy dataset $\D$ with $\D(R(a,b))=1$ and $\D(A)<1$ for all other atoms $A$ for which $\D(A)$ is defined.
Moreover, consider a guarded \tdat program $\prog$ which, among others, contains a rule $r$ of the form
$
    R(x,y) \rightarrow \exists z\, R(y,z).
$
A truth-greedy chase of $\prog$ and $\D$ applies this single rule infinitely many times before applying any other rule in \prog. 
Since other rules are not applied, it is possible that some triggers remain active infinitely.
Hence  
the fairness condition---\Cref{cond.fairness} in the definition of a chase---does not hold.
\end{example}
\noindent The question of how to decide entailment in \tdat with  infinite chases presents an interesting research challenge for future work. Decidability in guarded and warded \datex, in particular, follows from the fact that the infinite part of the chase only repeats certain patterns. 
In future we will aim to exploit this idea in the fuzzy setting of \tdat, to introduce adequate fuzzy chase procedures and establish computational complexity results for guarded and warded fragments of \tdat.

 \section*{Acknowledgements}
 Matthias Lanzinger
 acknowledges support by the Royal Society “RAISON DATA” project
 (Reference No. RP\textbackslash{}R1\textbackslash{}201074) and by the Vienna Science and Technology Fund (WWTF) [10.47379/ICT2201, 10.47379/VRG18013, 10.47379/NXT22018].
 Przemysław A Wałęga was supported bu
the EPSRC projects OASIS (EP/S032347/1), ConCuR
(EP/V050869/1) and UK FIRES (EP/S019111/1), as well as SIRIUS Centre for Scalable Data Access and Samsung Research UK.
 Stefano Sferrazza was supported by the Vienna Science and Technology Fund (WWTF) [10.47379/VRG18013]. 
 Georg Gottlob is a Royal Society Research Professor and
 acknowledges support by the Royal Society in this role through the
 “RAISON DATA” project (Reference
 No. RP\textbackslash{}R1\textbackslash{}201074).
For the purpose of Open Access,
the authors have applied a CC BY public copyright licence
to any Author Accepted Manuscript (AAM) version arising
from this submission.

\bibliographystyle{plain}
\bibliography{references}

\begin{thebibliography}{10}

\bibitem{DBLP:books/aw/AbiteboulHV95}
Serge Abiteboul, Richard Hull, and Victor Vianu.
\newblock {\em {Foundations of Databases}}.
\newblock Addison-Wesley, 1995.

\bibitem{achs1995fuzzy}
{\'A}gnes Achs and Attila Kiss.
\newblock Fuzzy extension of {D}atalog.
\newblock {\em Acta Cybernetica}, 12(2):153--166, 1995.

\bibitem{alviano2023generative}
Mario Alviano, Matthias Lanzinger, Michael Morak, and Andreas Pieris.
\newblock Generative {D}atalog with stable negation.
\newblock In {\em Proceedings of {PODS}}, pages 21--32, 2023.

\bibitem{amel2019shallow}
Kay~R Amel.
\newblock From shallow to deep interactions between knowledge representation,
  reasoning and machine learning.
\newblock In {\em Proceedings of SUM}, pages 16--18, 2019.

\bibitem{baget2011rules}
Jean-Fran{\c{c}}ois Baget, Michel Lecl{\`e}re, Marie-Laure Mugnier, and Eric
  Salvat.
\newblock On rules with existential variables: Walking the decidability line.
\newblock {\em Artificial Intelligence}, 175(9-10):1620--1654, 2011.

\bibitem{bardin2023machine}
S{\'e}bastien Bardin, Somesh Jha, and Vijay Ganesh.
\newblock Machine learning and logical reasoning: The new frontier ({D}agstuhl
  {S}eminar 22291).
\newblock In {\em Dagstuhl Reports}, 2023.

\bibitem{beeri1981implication}
Catriel Beeri and Moshe~Y Vardi.
\newblock The implication problem for data dependencies.
\newblock In {\em Automata, Languages and Programming: Eighth Colloquium Acre
  (Akko)}, pages 73--85, 1981.

\bibitem{DBLP:journals/is/BellomariniBGS22}
Luigi Bellomarini, Davide Benedetto, Georg Gottlob, and Emanuel Sallinger.
\newblock Vadalog: {A} modern architecture for automated reasoning with large
  knowledge graphs.
\newblock {\em Information Systems}, 105:101528, 2022.

\bibitem{DBLP:journals/tods/BergerGPS22}
Gerald Berger, Georg Gottlob, Andreas Pieris, and Emanuel Sallinger.
\newblock The space-efficient core of vadalog.
\newblock {\em {ACM} Trans. Database Syst.}, 47(1):1:1--1:46, 2022.

\bibitem{bobillo2008fuzzydl}
Fernando Bobillo and Umberto Straccia.
\newblock fuzzy{DL}: An expressive fuzzy description logic reasoner.
\newblock In {\em Proceedings of IEEE International Conference on Fuzzy
  Systems}, pages 923--930, 2008.

\bibitem{DBLP:journals/mst/CalauttiP21}
Marco Calautti and Andreas Pieris.
\newblock Semi-oblivious chase termination: The sticky case.
\newblock {\em Theory of Computing Systems}, 65(1):84--121, 2021.

\bibitem{DBLP:journals/jair/CaliGK13}
Andrea Cal{\`{i}}, Georg Gottlob, and Michael Kifer.
\newblock {Taming the infinite chase: Query answering under expressive
  relational constraints}.
\newblock {\em Journal of Artificial Intelligence Research}, 48:115--174, 2013.

\bibitem{cali2009general}
Andrea Cal{\`\i}, Georg Gottlob, and Thomas Lukasiewicz.
\newblock A general {D}atalog-based framework for tractable query answering
  over ontologies.
\newblock In {\em Proceedings of {PODS}}, pages 77--86, 2009.

\bibitem{DBLP:conf/rr/CaliGP10}
Andrea Cal{\`{\i}}, Georg Gottlob, and Andreas Pieris.
\newblock Query answering under non-guarded rules in datalog+/-.
\newblock In {\em Proceedings of Web Reasoning and Rule Systems}, pages 1--17,
  2010.

\bibitem{DBLP:journals/fss/CornejoLM18}
Maria~Eugenia Cornejo, David Lobo, and Jes{\'{u}}s Medina.
\newblock Syntax and semantics of multi-adjoint normal logic programming.
\newblock {\em Fuzzy Sets and Systems}, 345:41--62, 2018.

\bibitem{de2020statistical}
Luc De~Raedt, Sebastijan Duman{\v{c}}i{\'c}, Robin Manhaeve, and Giuseppe
  Marra.
\newblock From statistical relational to neuro-symbolic artificial
  intelligence.
\newblock {\em arXiv preprint arXiv:2003.08316}, 2020.

\bibitem{DBLP:journals/fss/Ebrahim01}
Rafee Ebrahim.
\newblock Fuzzy logic programming.
\newblock {\em Fuzzy Sets and Systems}, 117(2):215--230, 2001.

\bibitem{DBLP:journals/tnn/EklundK92}
Patrik Eklund and Frank Klawonn.
\newblock Neural fuzzy logic programming.
\newblock {\em IEEE Transactions on Neural Networks}, 3(5):815--818, 1992.

\bibitem{dataexchange}
Ronald Fagin, Phokion~G. Kolaitis, Ren{\'{e}}e~J. Miller, and Lucian Popa.
\newblock Data exchange: semantics and query answering.
\newblock {\em Theoretical Computer Science}, 336(1):89--124, 2005.

\bibitem{farahbod2012comparison}
Fahimeh Farahbod and Mahdi Eftekhari.
\newblock Comparison of different t-norm operators in classification problems.
\newblock {\em arXiv preprint arXiv:1208.1955}, 2012.

\bibitem{fodor2000fuzzy}
J{\'a}nos Fodor and Ronald~R Yager.
\newblock Fuzzy set-theoretic operators and quantifiers.
\newblock In {\em Fundamentals of Fuzzy Sets}, pages 125--193. 2000.

\bibitem{garcez2019neural}
Artur~d'Avila Garcez, Marco Gori, Luis~C Lamb, Luciano Serafini, Michael
  Spranger, and Son~N Tran.
\newblock Neural-symbolic computing: An effective methodology for principled
  integration of machine learning and reasoning.
\newblock {\em arXiv preprint arXiv:1905.06088}, 2019.

\bibitem{garcez2023neurosymbolic}
Artur~d’Avila Garcez and Luis~C Lamb.
\newblock Neurosymbolic {AI}: The 3rd wave.
\newblock {\em Artificial Intelligence Review}, pages 1--20, 2023.

\bibitem{DBLP:conf/kr/GottlobLP14}
Georg Gottlob, Thomas Lukasiewicz, and Andreas Pieris.
\newblock Datalog+/-: Questions and answers.
\newblock In {\em Proceedings of {KR}}, 2014.

\bibitem{gottlob2014datalog}
Georg Gottlob, Thomas Lukasiewicz, and Andreas Pieris.
\newblock Datalog+/-: Questions and answers.
\newblock In {\em Proceedings {KR}}, 2014.

\bibitem{DBLP:journals/fuin/GrahneO18}
G{\"{o}}sta Grahne and Adrian Onet.
\newblock Anatomy of the chase.
\newblock {\em Fundamenta Informaticae}, 157(3):221--270, 2018.

\bibitem{DBLP:books/kl/Hajek98}
Petr H{\'{a}}jek.
\newblock {\em {Metamathematics of Fuzzy Logic}}, volume~4 of {\em Trends in
  Logic}.
\newblock Kluwer, 1998.

\bibitem{DBLP:journals/tfs/IranzoS18}
Pascual~Juli{\'{a}}n Iranzo and Fernando S{\'{a}}enz{-}P{\'{e}}rez.
\newblock A fuzzy {D}atalog deductive database system.
\newblock {\em {IEEE} Transactions on Fuzzy Systems}, 26(5):2634--2648, 2018.

\bibitem{tnormsbook}
Erich{-}Peter Klement, Radko Mesiar, and Endre Pap.
\newblock {\em Triangular Norms}, volume~8 of {\em Trends in Logic}.
\newblock Springer, 2000.

\bibitem{lanzinger2022mvdatalog}
Matthias Lanzinger, Stefano Sferrazza, and Georg Gottlob.
\newblock {MV-D}atalog+-: Effective rule-based reasoning with uncertain
  observations, 2022.

\bibitem{lanzinger2022new}
Matthias Lanzinger, Stefano Sferrazza, and Georg Gottlob.
\newblock New perspectives for fuzzy {D}atalog (extended abstract).
\newblock In {\em Proceedings of {Datalog} 2.0}, pages 42--47, 2022.

\bibitem{luciano2022logic}
FBK Luciano~Serafini, Samy Badreddine, AI~Sony, and Japan~Ivan Donadello.
\newblock Logic tensor networks: Theory and applications.
\newblock {\em Neuro-Symbolic Artificial Intelligence: The State of the Art},
  342:370, 2022.

\bibitem{lukasiewicz2008managing}
Thomas Lukasiewicz and Umberto Straccia.
\newblock Managing uncertainty and vagueness in description logics for the
  semantic web.
\newblock {\em Journal of Web Semantics}, 6(4):291--308, 2008.

\bibitem{DBLP:journals/ai/ManhaeveDKDR21}
Robin Manhaeve, Sebastijan Dumancic, Angelika Kimmig, Thomas Demeester, and
  Luc~De Raedt.
\newblock {Neural probabilistic logic programming in DeepProbLog}.
\newblock {\em Artificial Intelligence}, 298:103504, 2021.

\bibitem{marra2019lyrics}
Giuseppe Marra, Francesco Giannini, Michelangelo Diligenti, and Marco Gori.
\newblock Lyrics: A general interface layer to integrate logic inference and
  deep learning.
\newblock In {\em Proceedings of ECML PKDD}, pages 283--298, 2019.

\bibitem{DBLP:conf/epia/MedinaOV01}
Jes{\'{u}}s Medina, Manuel Ojeda{-}Aciego, and Peter Vojt{\'{a}}s.
\newblock A procedural semantics for multi-adjoint logic programming.
\newblock In {\em Proceedings of {EPIA}}, pages 290--297, 2001.

\bibitem{DBLP:journals/cacm/Miller95}
George~A Miller.
\newblock {WordNet}: A lexical database for {E}nglish.
\newblock {\em Communications of the ACM}, 38(11):39--41, 1995.

\bibitem{serafini2016learning}
Luciano Serafini and Artur~S d’Avila Garcez.
\newblock Learning and reasoning with logic tensor networks.
\newblock In {\em Proceedings of {AIxIA}}, pages 334--348, 2016.

\bibitem{skryagin2022neural}
Arseny Skryagin, Wolfgang Stammer, Daniel Ochs, Devendra~Singh Dhami, and
  Kristian Kersting.
\newblock Neural-probabilistic answer set programming.
\newblock In {\em Proceedings of {KR}}, pages 463--473, 2022.

\bibitem{straccia2001reasoning}
Umberto Straccia.
\newblock Reasoning within fuzzy description logics.
\newblock {\em Journal of Artificial Intelligence Research}, 14:137--166, 2001.

\bibitem{yang2023neurasp}
Zhun Yang, Adam Ishay, and Joohyung Lee.
\newblock {NeurASP}: Embracing neural networks into answer set programming.
\newblock {\em arXiv preprint arXiv:2307.07700}, 2023.

\bibitem{DBLP:journals/chinaf/ZhangHX06}
Xiaohong Zhang, Huacan He, and Yang Xu.
\newblock A fuzzy logic system based on {S}chweizer-{S}klar t-norm.
\newblock {\em Science in China Series F: Information Sciences},
  49(2):175--188, 2006.

\end{thebibliography}

\appendix

\section{Details for \Cref{sec:negation}}\label{ap:negation}
We follow \cite{DBLP:books/aw/AbiteboulHV95} in the following definitions as well as the line or argumentation.
We will use a slightly different notation that in the main body of the paper, which is more convenient for proofs.

The \emph{extensional predicates} ($\edb(\Pi)$) of a \tdatu program $\Pi$ are those relation symbols that only occur in rule bodies of $\Pi$. The \emph{intensional predicates} ($\idb(\Pi)$) are those relation symbols that occur in the head of some rule of $\Pi$.
We know by \Cref{chase.embed} that every truth-greedy chase for \tdat programs (and thus also for semi-positive \tdatu programs) is finite. We write $\mathsf{tgc}(\Pi,\D)$ for the result of applying a truth-greedy chase\footnote{Recall that $\semiob$- and $\restr$-chase are equivalent without existential quantifiers.} to $\Pi$ and $\D$. 
Since $\mathsf{tgc}(\Pi,\D)$ is always finite, we will treat $\mathsf{tgc}(\Pi,\D)$ as a fuzzy dataset. As above, for $G \in \gatom$, we write $G \in \D$ to mean that $\D(G)$ is defined.

The notion of semipositive programs is lifted from Datalog$^\neg$. By convention we assume everywhere in the following that $\D$ is defined only for extensional predicates.
\begin{proposition}[cf. Theorem 15.2.2~\cite{DBLP:books/aw/AbiteboulHV95}]
\label{semipos}
Let $\Pi$ be a semipositive program and  $\D$ a fuzzy dataset. Then there is a unique fuzzy universal model of $\Pi$ and $\D$.
\end{proposition}
\begin{proof}[Proof sketch]
    For each unary operator $U$ and $R \in \edb(\Pi)$, let $R^U$ be a new relation with $\D'(R^U(\mathbf{c})) = U(\D(R(\mathbf{c})))$ for every ground atom where $\D(R(\mathbf{c}))$ is defined. Let $\D'(G)=\D(G)$ for all other $G \in \D$. Let $\Pi'$ be the \tdat program obtained by replacing every instance of $U\,R$ in the body of a rule with $R^U$. 
    The unique fuzzy universal model $\I$ of $\Pi'$ and $\D'$ (recall \Cref{minimalmodel}) is the desired unique fuzzy universal model of $\Pi$ and $\D$.
\end{proof}
Accordingly, we use the construction of $\Pi'$ and $\D'$ in the proof sketch of \Cref{semipos} also to define the semantics of a chase on semipositive programs: the result of applying the chase to semipositive program $\Pi$ and dataset $\D$ is the result of applying the chase to $\Pi'$ and $\D'$.
We write $\Pi(\D)$ for the unique fuzzy universal model of $\Pi$ and $\D$ from \Cref{semipos}.

A \emph{stratification} of a \tdatu program $\Pi$ is a sequence of \tdatu programs $\Pi_1,\dots,\Pi_n$ such that for some mapping $\zeta : \idb(\Pi) \to [n]$,
\begin{enumerate}[noitemsep,topsep=3pt,label=(\roman*)]
    \item $\{\Pi_1,\dots,\Pi_\ell\}$ is a partition of $\Pi$.
    \item for each relation symbol $R \in \idb(\Pi)$, all rules with $R$ in head are in $\Pi_{\zeta(R)}$.
    \item If there is a rule with $R$ in its head and $R' \in \idb(\Pi)$ in its body, then $\zeta(R') \leq \zeta(R)$. \label{strat3}
    \item If there is a rule with $R$ in its head and $U\, R'$ in its body, for some unary operator $U$, then $\zeta(R') < \zeta(R)$. \label{strat4}
\end{enumerate}
A program is \emph{stratifiable} if it has a stratification.

The semantics of a \tdatu program $\Pi$ is defined in terms of a procedure. Let $\zeta = P_1,\dots,P_n$ be a stratification of $\Pi$. Note that every $\Pi_i$ is semipositive. Define
\begin{align}
    & \D_0 = \D \\
    & \D_i = \mathsf{tgc}(\Pi_i,\D_{i-1}) \label{stratdef}
\end{align}
We denote $\D_n$ obtained from this procedure as $\zeta(\Pi,\D)$, and refer to it the semantics of a \tdatu program under $\zeta$.
We say that two stratifications of a \tdatu program are \emph{equivalent} if they have the same semantics for all fuzzy datasets.

For Datalog $^\neg$ it is well known that all stratifications lead to equivalent semantics~\cite{DBLP:books/aw/AbiteboulHV95}. The semantics of the formalism only affect a single key lemma, that we reprove below for our setting.
\begin{lemma}
    \label{semiposstrat}
    Let $\Pi$ be a semipositive \tdatu program and $\zeta$ a stratification for $\Pi$. Then $\Pi(\D)=\zeta(\D)$ for every fuzzy dataset $\D$.
\end{lemma}
\begin{proof}%
    Recall that we the semantics of semipositive programs are the result of applying the truth-greedy chase on $\Pi'$ and $\D'$ as defined in the proof of \Cref{semipos}.
    Let $s_i$ be the truth-greedy chase sequence of (\ref{stratdef}). It is enough to observe that their concatenation $s^* = s_1 s_2 \cdots s_n$  is a finite chase of $\Pi'$ and $\D'$. By \Cref{universalmodel} and \Cref{minimalmodel} this will produce the unique fuzzy universal model from \Cref{semipos}.

     To see that the concatenation is indeed a finite chase of $\Pi'$ and $\D'$ we need to only check that there is no active trigger at the end of $s^*$. We argue via induction on the the concatenation of the first $i\leq n$ sequences $s^{(i)} = s_1\cdots s_i$ that in the interpretation obtained from applying $s^{(i)}$ to $\Pi'$ and $\D'$, no trigger for rules in $\Pi_1,\dots,\Pi_i$ is active. First, the trigger cannot be for a rule in $\Pi_i$ as it is the result of applying the chase on $\D_{i-1}$. Second, by inductive assumption no triggers from $\Pi_1,\dots,\Pi_{i-1}$ were active after $s^{(i-1)}$. But by \Cref{strat3} and \Cref{strat4}, the heads of $\Pi_i$ cannot occur in the bodies of $\Pi_j$ with $j < i$. Therefore, the addition of $s_i$ also cannot make any trigger for rules in $\Pi_1,\dots,\Pi_{i-1}$ active.
\end{proof}

\begin{theorem}
    Let $\Pi$ be a stratifiable \tdatu program. All stratifications of $\Pi$ are equivalent.
\end{theorem}
\begin{proof}[Proof sketch]
    The proof of Theorem~15.2.10 in~\cite{DBLP:books/aw/AbiteboulHV95} holds also for our setting by simply replacing the key Lemma~15.2.9 there with \Cref{semiposstrat}.
\end{proof}

That is, we can use any stratification to compute the semantics of \tdatu programs, just as for stratified Datalog $^\neg$. Computing a stratification $\zeta$ is possible in polynomial time in the size of a dataset (see also Proposition 15.2.7~\cite{DBLP:books/aw/AbiteboulHV95}) and thus computing the truth degree of  $G \in \gatom$ in $\zeta(\D)$ is also polynomial (cf. \Cref{tDatalog.ptime}).
Thus, entailment in \tdatu is in \ptime for data complexity, as stated in  \Cref{strat:complex}. \end{document}